\documentclass{article}
\usepackage{natbib}
\usepackage{hyperref}

\usepackage[utf8]{inputenc}

\usepackage{fullpage}
\usepackage{amsfonts}       
\usepackage{nicefrac}       
\usepackage{microtype}      
\usepackage[dvipsnames]{xcolor}         

\usepackage{amsmath}
\usepackage{amsthm}
\usepackage{nicefrac}
\usepackage{amsfonts}
\usepackage[capitalise]{cleveref}
\usepackage{enumerate}
\usepackage{algorithm}
\usepackage{algorithmic}

\usepackage{wrapfig}
\usepackage{graphicx}
\usepackage{xcolor}
\usepackage{caption}
\usepackage{subcaption}
\usepackage{thmtools,thm-restate}

\newcommand{\cA}{\mathcal{A}}
\newcommand{\cV}{\mathcal{V}}
\newcommand{\cM}{\mathcal{M}}
\newcommand{\cO}{\mathcal{O}}
\newcommand{\cS}{\mathcal{S}}
\newcommand{\cC}{\mathcal{C}}
\newcommand{\rT}{\mathrm{T}}
\newcommand{\rC}{\mathrm{C}}
\newcommand{\R}{\mathbb{R}}
\newcommand{\Z}{\mathbb{Z}}

\DeclareMathOperator*{\argmin}{arg\,min}

\newtheorem{theorem}{Theorem}
\newtheorem{claim}{Claim}

\newtheorem{lemma}{Lemma}

\title{There is no Accuracy-Interpretability Tradeoff in Reinforcement Learning for Mazes}

\author{
  {Yishay Mansour}\thanks{Tel Aviv University and Google Research. 
}
  \and
  {Michal Moshkovitz}\thanks{Tel Aviv University. }
   \and
  {Cynthia Rudin}\thanks{Duke University.}
}

\date{}

\begin{document}

\maketitle

\begin{abstract}
Interpretability is an essential building block for trustworthiness in reinforcement learning systems. However, interpretability might come at the cost of deteriorated performance, leading many researchers to build complex models. Our goal is to analyze the cost of interpretability. We show that in certain cases, one can achieve policy interpretability while maintaining its optimality. We focus on a classical problem from reinforcement learning: mazes with $k$ obstacles in $\mathbb{R}^d$. We prove the existence of a small decision tree with a linear function at each inner node and depth $O(\log k+2^d)$ that represents an \emph{optimal} policy. Note that for the interesting case of a constant $d$, we have $O(\log k)$ depth. Thus, in this setting, there is no accuracy-interpretability tradeoff. To prove this result, we use a new ``compressing'' technique that might be useful in additional settings. 
\end{abstract}

\section{Introduction}

Deploying interpretable policies for reinforcement learning (RL) tasks, where the policy is humanly-understandable, is an important aspiration. However, one might worry that restricting only to interpretable policies might come at a price. Thus, a fundamental question arises: what is the price of interpretability? There are complex domains for which uninterpretable policies are known while an interpretable policy has not been found, for instance, protein folding  \citep[AlphaFold,][]{jumper2021highly} or games like Go \citep{silver2016mastering}, chess \citep{silver2018general} and Atari \citep{mnih2013playing}. However, the price of interpretability is unknown. Perhaps, an interpretable policy with good performance exists for these domains, only that no one has found one yet. 

To understand whether the same policy can be both optimal and interpretable, we return to one of the first domains RL was studied in --- mazes. In this domain, the agent plans a policy that takes it to a goal state. The agent moves up, down, left, or right at each state while not running into any obstacles. 

Mazes are significant both from the theoretical perspective we consider here and from the practical perspective: numerous real-world robotic problems can be cast as mazes. 
Examples include robots traversing warehouse floors \citep{everett1995real};
robots assembling products in factories \citep{fukuda1985flexibility} could also be modeled as a multi-dimensional maze, where the obstacles are constraints on how parts are assembled. In a sense, mazes can be used to model the decisions of AI-aided doctors providing drugs to critically ill patients over the timeframe they are in the intensive care unit of a hospital \citep{liu2017deep}; patients cannot receive more than a certain dose of a drug at a time, which can be viewed as an obstacle in a maze of doses that could be assigned to the patient.

In this paper, we analyze the price of interpretability for mazes. There are several types of interpretable models, but the gold standard is decision trees, which is our focus \citep[see, e.g.,][]{roth2019conservative, silva2020optimization}. The input to the decision tree is the current state, and the output, at the leaf, is the action to perform. At each inner node, there is a linear function of the current state. The complexity of the policy is the number of steps until reaching a decision, i.e., the depth of the tree. We show the existence of an interpretable \emph{optimal} policy and show an efficient algorithm for finding such a policy. 
The algorithm's output is a decision tree which implements an optimal policy, which guarantees that the agent reaches the goal state at minimum cost.
The depth of the decision tree is provably bounded, but the depth is not necessarily the smallest.
Namely, we prove the following. 
\begin{theorem}[main theorem, informal]\label{thm:rd_maze_adaptive_interpretable_policy_intro}
For every maze with $d$ dimensions and $k$ obstacles, there is an optimal interpretable (i.e., sparse decision tree) policy with depth $O(\log k+2^d)$, which  for a constant $d$  is $O(\log k)$. Additionally, if the starting state is sparse with at most $n$ non-zero coordinates, then the policy depth is $O(\log k+2^n)$.
\end{theorem}

\begin{wrapfigure}[16]{R}{0.45\textwidth}
    \centering
     \includegraphics[width=0.44\textwidth,trim={5cm 1cm 7cm 1cm},clip]{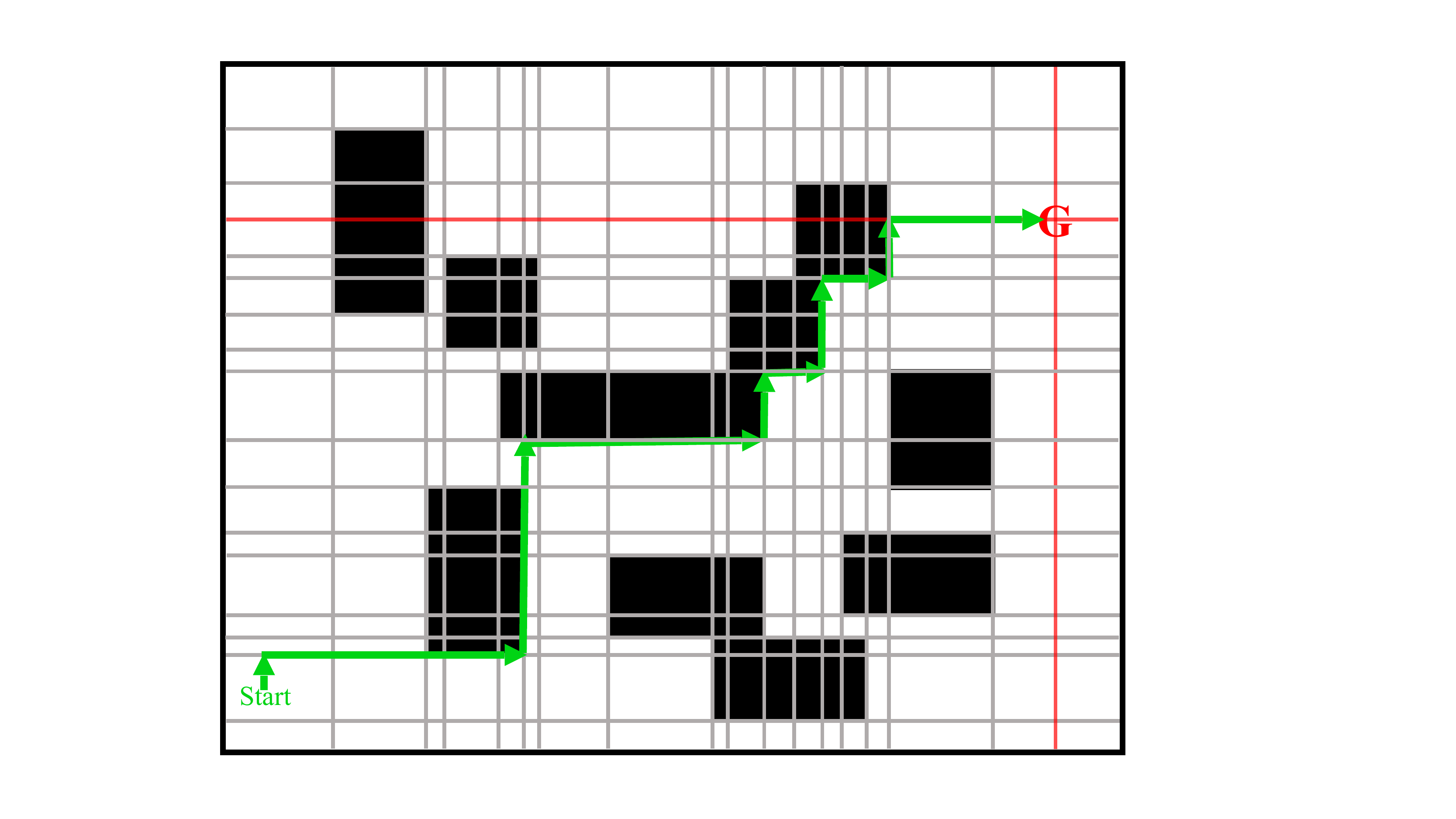}
    \label{fig:mdp_maze_intro}
\end{wrapfigure}
To prove the theorem, we find a small subset of states, which depends only on $k$ and $d$, and prove that there is a structured optimal policy that traverses from the start state to one of those states, and continue traversing between states in this set until reaching the goal state. This observation reduces the original maze problem to only those states, and standard algorithms can solve mazes with this smaller number of states. This ``compressing'' technique for finding interpretable models might assist in additional settings. See the figure on the right for an example of the agent's trajectory from \textcolor{ForestGreen}{Start} to \textcolor{red}{G}, using the new technique. The small subset of states are the intersections of the gray lines.

An advantage of the policy we construct is that the agent does not need to traverse the entire decision tree when encountering a new state. More specifically, the agent traverses the tree once to find the action for the starting state, and all the following actions are found deterministically in time $O(1)$. Additionally, we consider a ``uniform noise'' scenario, where the agent action might ``fail.'' More precisely, there is a fixed probability that the agent stays in the same place regardless of the action performed.  We show that under this specific noise model the optimal policy does not change, and hence our results immediately carry over. (See Appendix~\ref{sec:noise}.) 


\subsection*{Summary of Contribution}
\begin{itemize}
    \item \textbf{Main conceptual result}: There is an optimal interpretable policy for mazes, so no need to tradeoff interpretability with accuracy in mazes. 
    \item \textbf{Quantitatively}: For mazes in  $d$ dimension and $k$ obstacles, there is an optimal policy defined by a decision tree with depth $O(\log k +2^d)$. For the interesting case of 2-d and 3-d mazes this gives $O(\log k)$ depth. If the starting state is sparse with at most $n$ non-zero coordinates, then the depth of the decision tree  is $O(\log k+2^n)$. The same results hold in the ``uniform noise'' model.
\item \textbf{Running the agent:} The interpretable policy allows the agent to traverse efficiently in the decision tree --- no need to traverse the entire tree upon reaching every new state. For the first step the agent will spend time  $O(\log k +2^d)$, and in following steps it will spend only $O(1)$ time.
\item \textbf{Proof methodology:} We suggest a new ``compressing'' technique to build (and prove) an interpretable policy.
\end{itemize}

\subsection*{Related Work} 


\paragraph{Interpretable policies.} 
Several works 
\citep{roth2019conservative,bastani2018verifiable,liu2018toward,verma2018programmatically,coppens2019distilling} build interpretable policies, either tree-based policies or ``programmatically interpretable.'' Unfortunately, these works are mostly heuristics, i.e., they do not provide guarantees for their quality. Alternatively, \citet{silva2020optimization} focuses on a simple environment (simplified 1-d pole problem) with a straightforward optimal policy. 
In contrast, we provide provable guarantees for a general $d$-dimension maze problem. 


\paragraph{Interpretability in machine learning.} 
There are many works on machine learning interpretability, see \cite{rudin2022interpretable}, and most of them are heuristics. Our focus is on \emph{provably} interpretable  machine learning and mainly the price of interpretability. 
\citet{dasgupta2020explainable,makarychev2021explainable}, studied this price for unsupervised learning, specifically, for $k$-means clustering. The tradeoff is between interpretability and the quality of the model, measured by the $k$-means cost. These works explored the price, by proving lower and upper bounds on the $k$-means cost as a function of the tree's complexity or when it is the smallest possible.  \citet{moshkovitz2021connecting} proved an upper bound on the price of interpretability for supervised learning, under distributional assumptions. 
In this paper, we explore the price of interpretability for reinforcement learning.

\paragraph{Interpretable policies with human-in-the-loop.}  Different works find interpretable policies by utilizing additional external information. For example, \citet{shu2017hierarchical} builds hierarchical policies composed of different subtasks, where a human provides the subtasks. As the human feedback is limited, the optimality of the policies can be guaranteed. 
In this paper, we find an interpretable policy based solely on the Markov decision process.

\paragraph{Post-hoc policies explanations.} The aim of interpretability is to make policies transparent. Another approach to gain (partial) transparency is by adding explanations to preexisting policies.  \citet{amir2018highlights, sequeira2020interestingness,juozapaitis2019explainable} explored different types of explanations, differing in what and how to explain (e.g., explaining rewards, states, or trajectories). The disadvantages of post-hoc explanations for critical systems is elaborated by \citet{rudin2019stop}.

\section{Setup and Problem Formulation} 
\subsection{Maze Definition}
 \begin{wrapfigure}[20]{r}{.4\textwidth}
    \centering
    \includegraphics[width=0.34\textwidth,trim={7cm 5cm 14cm 2cm},clip]{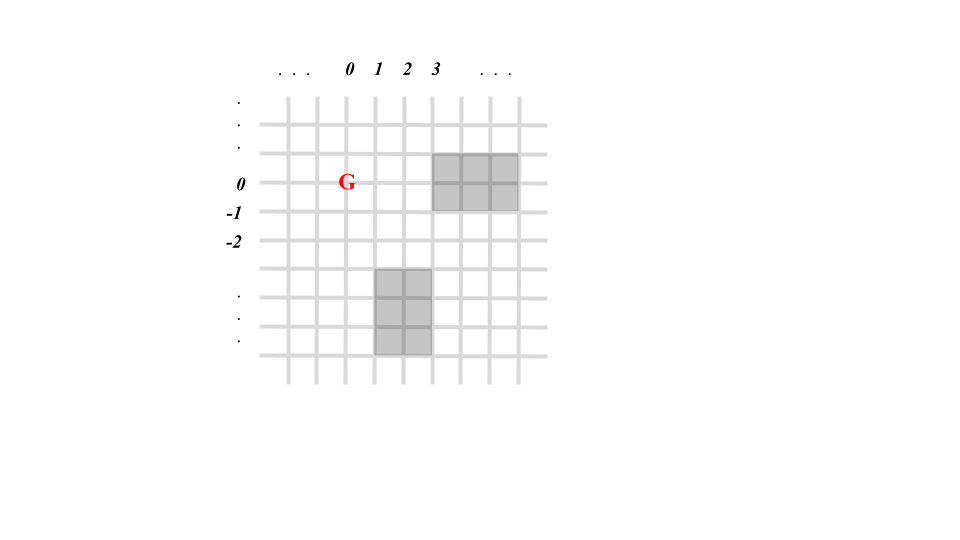}
    \caption{2D-maze. 
    Goal state at $(0,0)$, denoted by  \textcolor{red}{G}. Grey: two obstacles.
    Policy navigates the agent from current square to the goal state.}
    \label{fig:mdp_maze}
\end{wrapfigure}
An \emph{obstacle} is an axis-aligned hyperrectangle $R_{a,b}$, defined by two vectors\footnote{for convenience, and w.l.o.g., the edges of the hyperrectangles are not part of the obstacle.} $a,b\in \R^d$ as $$R_{a,b}=\{x\in \R^d : \forall i,\; a_i <  x_i <  b_i\},$$ where $x_i$ is the $i$-th coordinate of $x$.
The $(k,d)$-\emph{Maze} is a Markov Decision Process (MDP) which is a tuple $\cM=\langle\cS, \cA, \rT, \rC \rangle$ defined by $k$ obstacles,  $\cO^1,\ldots,\cO^k.$ When $k$ and $d$ are clear from the context, we sometimes refer to the $(k,d)$-\emph{Maze} simply as the \emph{maze}. For ease of notation, we denote by $\cO$ the set of all points comprising all of the obstacles, $\cO = \cup_{i=1}^k \cO^i.$
The (infinite) state space is  $\cS=\Z^d$, where one of the states is the (absorbing) \emph{goal state}, $g\in \cS$.
There are $2d$ actions: $\cA=[d]\times\{\pm 1\}$. For $d=2$, the actions correspond to the four actions:  \texttt{up}, \texttt{down}, \texttt{right}, and  \texttt{left}. 

The transition function $\rT:\cS\times\cA\rightarrow\cS$ is deterministic defined by the action chosen. If the agent tries to go through an obstacle, or move from the goal state, it will stay in place.
 More formally, for a state $s$ and an action $(i,b)\in [d]\times \{\pm 1\}$ we first define $s'\in\Z^d$ as the move from $s$, in the $i$-th feature by $b$, i.e., $s'_i=s_i+b$ and $s'_j=s_j$ for all $j\neq i$. Then, the transition function is defined as

\[
   \rT(s,a) = \Bigg\{\begin{array}{lc}
        s\quad\quad & \text{if } s = g \text{ or } s'\in \cO\\
        s'\quad\quad & \text{otherwise }
        \end{array} 
  \]
The cost $\rC:\cS\times\cA\rightarrow\R$ is $1$ for each action, except actions at the goal state with cost $0$ for each action, i.e., $\rC(s,a)=1\cdot \mathbb{I}_{s\neq g}$, where $\mathbb{I}$ is the indicator function.

A policy is a function $\pi:\cS \rightarrow \cA$. 
Any start state $s_0\in\cS$ defines an episode $(s_0, s_1,\ldots, s_T)$ where $s_{t+1}=\rT(s_t,\pi(s_t))$. The episode terminates once $s_T=g$.  We call $s_0$ the \emph{starting state}.
The total cost of a policy for this episode is the sum of the costs throughout, which is the number of actions in the episode, i.e., $T$.    
The agent aims to minimize the total cost of all possible episodes. 
The value function, $V^{\pi}:\cS\rightarrow\R$, of  policy $\pi$ is the cost to reach the goal state from each state, when following policy $\pi.$ The value of an optimal policy is denoted by $V^{*}.$

\subsection{Interpretable Policy} 

An interpretable policy is defined by a decision tree. Each internal node contains a linear threshold function of the current state. Thus, it iteratively partitions the data, leading to leaves. The leaves are labeled by the actions that the policy executes in the current state. The policy's depth defines its complexity. In the paper, we aim to find an optimal interpretable policy with a small complexity. 

\subsubsection{Efficient Implementation} 
In its most straightforward implementation, the interpretable policy traverses the tree upon reaching every new state, which is costly. Our interpretable policy is more efficient. First, as long as the leaf remains the same when moving to another state, there is no need to traverse the tree again, and the same action applies. For that, we use the notation  \[ (i, b, \texttt{while condition}), \] which means that we use the action $(i, b)$ (i.e., change feature $i$ by $b$) repeatedly as long as \texttt{condition} is satisfied. 
Once the leaf of the starting state is located, the next leaf can be found in $O(1)$. In fact, after traversing the tree once to locate the starting state, the policy does not need to traverse the full tree from the top ever again.

\subsection{Problem Formulation and Main Theorem} Given $k$ obstacles and a goal state $g$, we aim to find an optimal policy $\pi^*:\cS\rightarrow\cA$ that is defined by a decision tree with small depth. 
One cannot apply standard methods like policy iteration, since it does not find interpretable policies. Nonetheless, we design an algorithm that finds an interpretable optimal policy.
We are now ready to state the main theorem. 
\begin{theorem}[main theorem, formal]\label{thm:rd_maze_adaptive_interpretable_policy}
For every $(k,d)$-maze, there is an optimal interpretable policy $\pi^*:\cS \rightarrow \cA$ with complexity $O(\log k + 2^d)$, which implies that for a constant $d$ it is $O(\log k)$. If the  starting state has at most $n$ non-zero features, then the complexity of the policy is $O(\log k+2^n)$.
\end{theorem}

\section{Gentle Start: The Case of $d=2$}\label{sec:gentle_start}

To better understand our main result, we start with the special case that the dimension is $d=2$.
In this case the points are in the plane and obstacles are axis-aligned rectangles. The main idea would be to identify a subset of points, which we call corners. The set of corners would be of size at most $O(k^2)$, and the optimal policy would go from the start state to a corner, and continue going between corners until it reaches the goal. The main theoretical challenge would be to prove that indeed there is such an optimal policy, and then to compute it. Finally, we would implement the policy using a decision tree of depth $O(\log k)$. 

\begin{figure}[!h]
\centering
\begin{subfigure}[t]{.45\textwidth}
  \centering
    \includegraphics[scale=0.37,trim={7cm 5cm 13cm 2cm},clip]{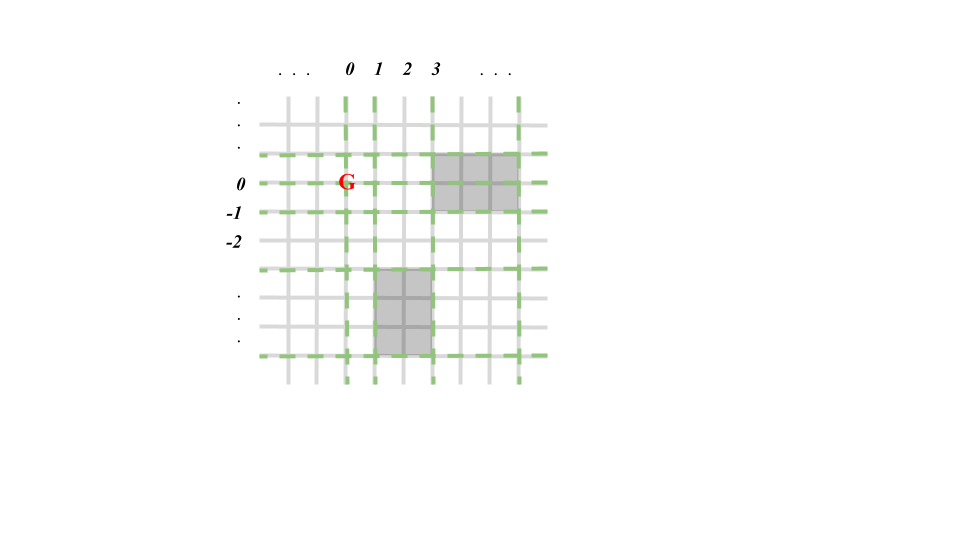}
  \caption{Horizontal and vertical lines}
  \label{fig:2d_the_grid}
\end{subfigure}%
\hfill
\begin{subfigure}[t]{.45\textwidth}
  \centering
  \includegraphics[scale=0.37,trim={7cm 5cm 13cm 2cm},clip]{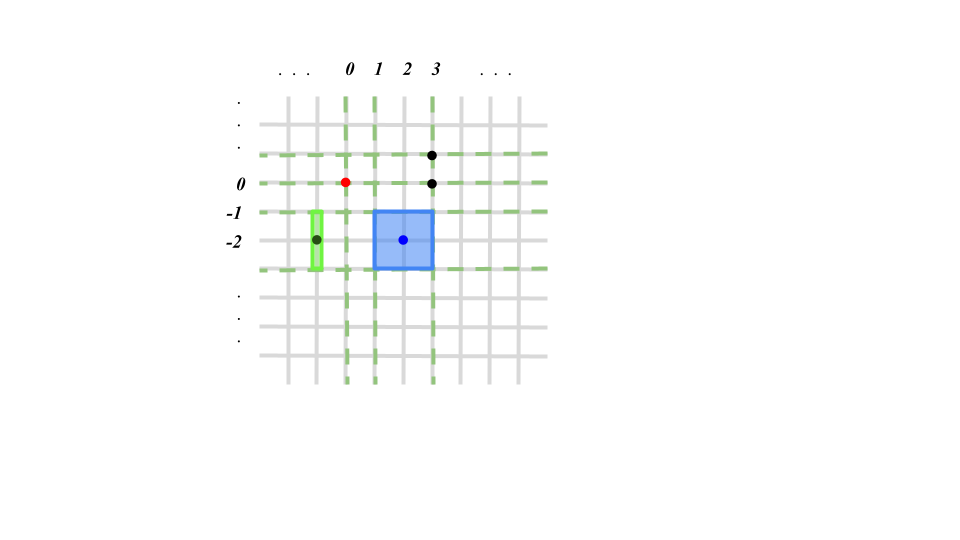}
  \caption{Surfaces and adjacent corners}
  \label{fig:2d_surface}
\end{subfigure}
\caption{Terminology. (a) Dashed lines: horizontal and vertical lines (b) Surface of the red point (corner) is the corner. Surface of the green point (on an edge) is the entire segment between two consecutive lines including the two corners. Surface of the blue point is the entire rectangle including the four edges and corners. Two black points are adjacent corners, both not adjacent to the red corner. 
}
\label{fig:test}
\end{figure}

We start by defining the set of points that we will call corners and related notions required for our presentation. We begin by introducing a few key terms. A horizontal (vertical) line is defined by an integer, e.g., the line $y=1$ is encoded by the integer $1$. 
An obstacle is defined by two horizontal and two vertical lines. For example, in \cref{fig:2d_the_grid} there are two obstacles. The top one is defined by two horizontal lines: $-1,1$, and two vertical lines $3,6$. The second obstacle is defined by the horizontal lines $-3,-6$ and vertical lines $1,3$. We denote an obstacle as two pairs $(h_1,v_1), (h_2,v_2)\in\R^2$ where $h_1<h_2$ are the horizontal lines and $v_1<v_2$ are the vertical lines. (Note that the points $(h_1,v_1)$ and $(h_2,v_2)$ are the two opposite corners of the obstacle rectangle.)

We now define $H$ to be the set of all $h_i$ of the obstacles, and $V$ the set of all the $v_i$ of the obstacles. We also add the goal state $(g_1,g_2)$, namely add $g_1\in H$ and $g_2\in V$.
The set of corners\footnote{For ease of presentation we assume that $-\infty,\infty\in H,V$, where $\infty$ (respectively $-\infty$) is larger (smaller) than any number;  and the segment $[a,\infty]$ ($[-\infty,a]$) is  the half-open interval $[a,\infty)$ (respectively $(-\infty,a]$).}  is $\cC= V\times H$ and the number of corners is $O(k^2).$
%
For example, the horizontal line $-6$ and the vertical line $3$ define the corner $(3,-6).$

Using sets $H$ and $V$ we can define the \emph{surface} of a point, which will be an axis-aligned rectangle which is obstacles free.
Fix a point $(x,y)\in\R^2$, its surface is defined by the following two segments, $X$ and $Y$. If $x\in V$, then $X = \{x\}$. Otherwise, $v^j < x < v^{j+1}$ and $X$ is the segment between $v^j$ and $v^{j+1}$, i.e.,  $X=[v^j, v^{j+1}]$. Similarly one can define $Y$. The surface of $(x,y)$ is  \[S((x,y)) = X \times Y.\] 
See example in \cref{fig:2d_surface}. From the definition of the surface we have that the surface of a corner $c$ is the corner itself, $S(c) = \{c\}$. 

The next claim proves that the agent can freely go from a state $s\in\cS$ to any point in its surface, $S(s)$, without encountering obstacles. Namely, for any $s'\in S(s)$ the agent can modify first the $x$ coordinate from $s_x$ to $s'_x$ and then from $s_y$ to $s'_y$, without hitting any obstacle in $\cO$. This will give us the freedom to leave the surface in many different states. We utilize it to implement the interpretable policy. The general proof for any dimension appears in \cref{sec:algorithm_building_policy}. 
\begin{claim}\label{clm:surface_free_d_2}
For any $s\in\R^2$, if $s\notin \cO$ then $S(s)\cap \cO=\emptyset$. 
\end{claim}
Another key term is \emph{adjacent corners}. Two corners are \emph{adjacent} if equal on one feature and the other feature has two consecutive lines in $H$ or $V$, while not having an obstacle between them, see \cref{fig:2d_surface}.

%

\subsection{The Existence of a Structured Optimal Policy}\label{sec:2_d_structured_policy}

    \begin{wrapfigure}[21]{r}{.55\textwidth}
    \centering
     \includegraphics[width=0.54\textwidth,trim={0cm 0cm 10cm 0cm},clip]{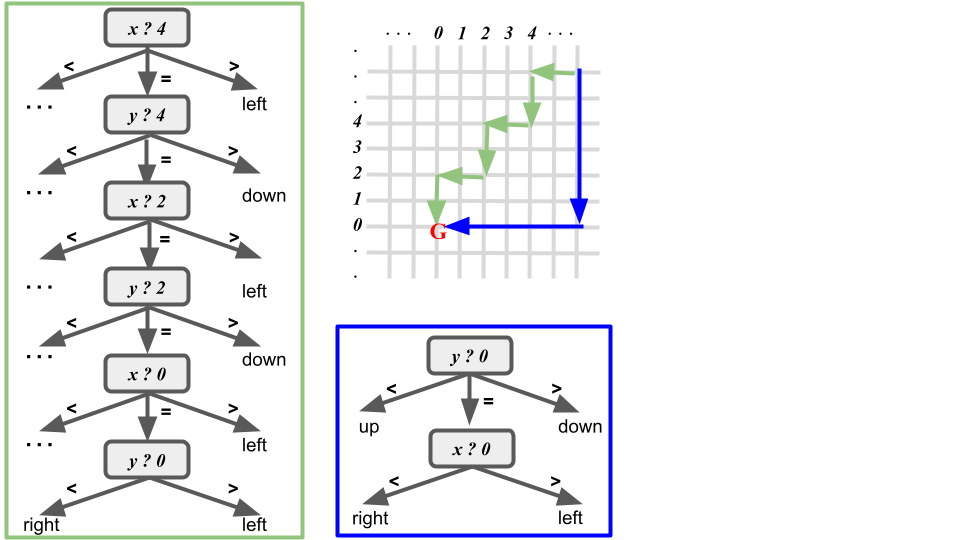}
     \caption{Two optimal policies with different decision tree depth. Due to space limitations, some parts of the green policy are omitted and replaced by ``...'' }
 \label{fig:rashomon_of_optimal_policy}
    \end{wrapfigure}
  
The set of optimal policies is not necessarily a singleton, as exemplified in \cref{fig:rashomon_of_optimal_policy}. 
When there are multiple optimal policies, there is flexibility in choosing among them. The next section uses this flexibility to construct one policy which has an interpretable (i.e., sparse and easy-to-understand) structure. 


We focus on a policy $\tilde{\pi}$, see \cref{alg:pitilde} where the agent moves from one corner to an adjacent one, until reaching the goal state. 
 The agent iteratively chooses its next move (i.e., which corner to go to next) using a new MDP $\cM'$ defined on the $O(k^2)$ corners. There are four actions, one for each adjacent corner. The optimal policy, $\pi^*_{\cM'},$ for the \emph{finite} maze problem, $\cM'$, can be found efficiently, e.g., using Dijkstra's algorithm \citep{cormen2022introduction} that finds shortest paths to the goal state.
 
   \begin{wrapfigure}[11]{r}{.5\textwidth}
   \begin{minipage}[tr]{.5\textwidth}
     \centering
 \begin{algorithm}[H]
\caption{Policy $\tilde{\pi}$}\label{alg:pitilde}
\begin{algorithmic}[1] 
\IF {$s$ is a corner} 
\STATE \textbf{return} $\pi^*_{\cM'}(s)$ 
\ELSE
\STATE  $F \leftarrow \text{ corners in } S(s)$ \STATE $c = \textrm{arg}\min_{c \in F}  V^*_{\cM'}(c) + \sum_{i=1}^2|s_i-c_i|$
\STATE \textbf{return} go to corner $c$
\ENDIF
\end{algorithmic}
\end{algorithm}
  \end{minipage}
  \end{wrapfigure}
 
 The formal proof that policy $\tilde{\pi}$ is an optimal one appears in \cref{sec:algorithm_building_policy}. 
 

\subsection{Building the Decision Tree}
Building the decision tree is composed of two steps. Recall that the input to the decision tree is $(x,y)$. First, the decision tree detects the surface the starting state $(x,y)$ is in. Second, given the detected surface, the decision tree chooses which action to take.  
These two steps are illustrated in \cref{fig:building_policy} for the starting state $(2,-2)$. The left decision tree finds the surface. The right decision tree continues and computes the optimal distance to the goal and the corresponding corner the agent should visit next (we elaborate more next).


\paragraph{Step 1: finding the surface.}
In the first step, we want to detect the location of the current state $(x, y)$ in the list of vertical lines, $V$. Namely, we want to find index $i$ such that  $v^i \leq x < v^{i+1}$. The binary search technique enables us to find this $j$ using a decision tree of depth $O(\log k).$ In \cref{fig:building_policy} (left), these are the first two levels of the decision tree. Similarly, we use the binary search technique over the horizontal lines and find an index $j$ such that $h^j\leq y < h^{j+1}$. 
At the end of this step each leaf corresponds to indexes $i, j$ and 
all points in the same leaf have exactly the same surface. 

\begin{figure}[!h]
\centering
     \includegraphics[width=0.9\linewidth]{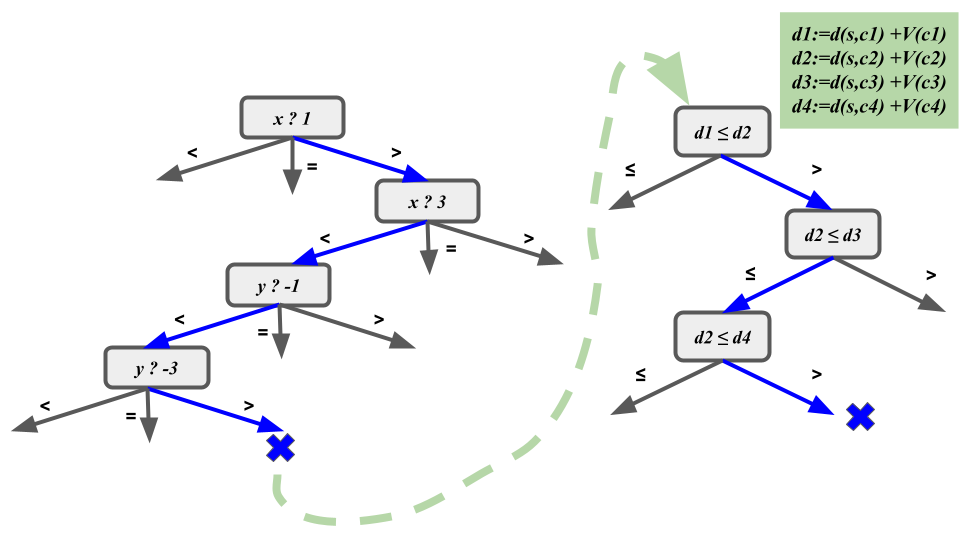}
     \caption{Building interpretable  policy. Illustrated for maze in \cref{fig:2d_surface}.  (left) The output of step 1 (\cref{alg:build_tree}). 
     Blue: the trajectory for state $(2,-2)$. (right) The output of step 2 (\cref{alg:build_direction_tree}), which finds best corner among $c1,c2,c3,c4$.}
     
 \label{fig:building_policy}
     \end{figure}
     
\paragraph{Step 2: taking an action.}
Our agent should go to one of the (at most) four corners in the starting state's surface. 
For each of the corners we compute the distance from the starting state to that corner plus the distance from the corner to the goal. (In \cref{fig:building_policy} we use $d(s,c_i)$ to be the distance to corner $c_i$ and $V(c_i)$ to be the distance from $c_i$ to the goal.)
We then select as the next corner as the corner that minimizes total distance.
Since \cref{alg:pitilde} is an optimal policy,
we are guaranteed that there is an optimal policy that goes to one of the corners. Showing that indeed \cref{alg:pitilde} is optimal is deferred to the general case and appears in \cref{thm:rd_value_function_of_large_and_small_MDP}


\paragraph{Putting it all together.} 
The decision tree we build implements the optimal policy $\tilde{\pi}$, see \cref{alg:pitilde}, and thus it is an optimal policy, see \cref{sec:2_d_structured_policy}. 
The policy complexity is controlled by the first step, which is $O(\log k).$
For general dimension, see next section.  
\section{Algorithm for Building an Optimal Interpretable Policy}\label{sec:algorithm_building_policy} 
In the rest of the section we describe the general algorithm that builds the interpretable policy in $d$ dimensions. Following \cref{sec:gentle_start}, the high-level approach is to construct a smaller MDP, $\cM'$, and use an optimal policy on $\cM'$ to construct the interpretable policy. In \cref{subsec:smaller_MDP} we construct $\cM'$ using the terminology introduced in \cref{subsec:direct_access} and \cref{subsec:the_grid}. Then, in \cref{subsec:the_algorithm} we describe the desired algorithm. Importantly, this will prove that there \emph{exists} an interpretable policy. Thus, there is no accuracy/interpretability tradeoff in
low dimensional mazes

\subsection{Direct Access}\label{subsec:direct_access}
We say that there is \emph{direct access} between a state $s$ and a state $s'$ if \texttt{DirectGoTo} algorithm can be applied to go from $s$ to $s'$. Note that if the \texttt{DirectGoTo} can be applied, this is a shortest path with smallest cost between them, $d(s,s')$, where $d$ is the Manhattan distance function defined as
$d(x,y)=\sum_{i=1}^d |x_i-y_i|.$
\begin{algorithm}
\caption{DirectGoTo}
\begin{algorithmic}[1] 
\STATE \textbf{Input:} $s,s'\in \cS$
\STATE \textbf{Output:} series of actions
\FOR {$i\in[d]$}
\IF {$s_i < s'_i$}
\STATE Take action $(i, +1,\texttt{ while } s_i < s'_i)$
\ELSIF {$s_i > s'_i$}
\STATE Take action $(i, -1, \texttt{ while } s_i >  s'_i)$
\ENDIF
\ENDFOR
\end{algorithmic}
\end{algorithm}

\subsection{Corners, Surfaces and Adjacent Corners}\label{subsec:the_grid}
In this section we introduce key terms like a \emph{corner} and a \emph{surface} that will be essential in the main proof. 
Recall that a maze is defined using a goal state $g\in\cS$ and $k$ obstacles $R_{a^1,b^1},\ldots, R_{a^k,b^k}$.  
We first define for each feature $i\in[d]$ a \emph{list of scalars} $L^i\subseteq \Z$, which come from the $i$th coordinate of each obstacle and the goal state:\footnote{For ease of presentation we assume that $-\infty,\infty\in L^i$ for each list; $\infty$ ($-\infty$) is larger (smaller) than any number;  and the segment $[a,\infty]$ ($[-\infty,a]$) is  the half-open interval $[a,\infty)$ ($(-\infty,a]$) and  $[-\infty,\infty]$ is the open interval $(-\infty,\infty)$.}
\[L^i =  \cup_{j=1}^k \{a^j_i,b^j_i\} \cup \{g_i\} \]
For example, for the maze in \cref{fig:mdp_maze}, $L^1=\{0,1,3,6\}$ and  $L^2=\{-1,0,1,3,6\}.$
 Denote $n_i=|L^i|.$ Assume without loss of generality that the elements in $L^i$ are sorted $L^i[1]\leq L^i[2]\leq  \ldots \leq L^i[n_i].$ 
%
A \emph{corner} is a vector $c\in\Z^d$ such that all of its coordinates are in the list, i.e., $\forall i\in [d]$, $c_i \in L^i.$ 
We denote the set of all corners by \[\cC=L^1\times\ldots\times L^d.\]
The total number of corners is $|\cC|=\prod_{i=1}^d n_i.$  
Two corners $c^1,c^2$ are \emph{adjacent} if they are the same in all features, except exactly one, $i$, where $c^1_i$ and $c^2_i$ are consecutive in $L^i$, and there is no obstacle between them. It implies that  there is  direct access between them. Clearly, each corner has at most $2d$ adjacent corners. 

Given a point $x\in\R^d$, we define its \emph{surface}, $S(x)$, as the set of points that can be reached from $x$ until (and including) encountering a value in one of the lists $L^i$. More formally, define for each feature $i$ the segment $S^i$ in the following way. If $L^i[j]< x_i < L^i[j+1]$ then $S^i = [L^i[j], L^i[j+1]]$ and if $x_i=L^i[j]$ then $S^i = [L^i[j]]$. The surface of $x$ is defined as 
\[ S(x) = S^1\times\ldots\times S^d.\]
For example, if $x$ is a corner then $S(x)=\{x\}.$
Any surface contains at most $2^d$ corners. 
A key advantage of this definition for surfaces is that if the agent reaches a point $x$, it can directly reach all of its surface, $S(x)$ without hitting an obstacle. In other words, if a point is not in $\cO$, then none of its surface is in $\cO$. 
\begin{lemma}\label{lem:surface_not_obstacle}
For any $x\in\R^d$, if $x\notin \cO$ then $S(x)\cap \cO=\emptyset$. 
\end{lemma}
\begin{proof}
Fix $x\notin \cO$ and $y\in S(x).$ We will show that $y\notin \cO$. Assume by contradiction that  $y\in \cO.$ Then, there is an obstacle $R_{a,b}$ such that for all $i\in[d]$, it holds that $a_i < y_i < b_i.$  Fix $i\in[d]$. If $x_i=L^i[j]$ then $y_i=x_i$ and $a_i < x_i < b_i$ If $L^i[j] < x_i < L^i[j+1]$, then, $L^i[j] \leq y_i \leq L^i[j+1]$. From the definition of the grid, $b_i\geq L^i[j+1]$ and $a_i\leq L^i[j]$ which implies that $a_i < x_i < b_i$. Thus, for all $i\in[d]$, $a_i < x_i < b_i$. I.e., $x\in R_{a,b}$, which is a contradiction.
\end{proof}

A corollary of \cref{lem:surface_not_obstacle} is that \texttt{DirectGoTo} can be applied to any state $s$ and $s'$, in its surface, $s'\in S(s)$, by advancing one feature at a time, while not running into any obstacle.

\begin{claim}\label{clm:direct_access}
Algorithm \texttt{DirectGoTo} implements the shortest path from $s$ to $s'$ in the following cases 
\begin{itemize}
    \item If $s\in \cS$ and $s'\in S(s)$
    \item If $s\in\cC$ and $s'\in\cC$ is an adjacent corner to $s$ 
\end{itemize}
\end{claim}
\begin{proof}
The algorithm is the shortest path, assuming it does not run into obstacles. If $s\in\cC$ and $s'\in\cC$ is an adjacent corner to $s$, \texttt{DirectGoTo} will not run into obstacles, by the definition of adjacent corners. If $s\in \cS$ and $s'\in S(s)$, then all states \texttt{DirectGoTo} traverses are in $S(s)$ and the claim follows from \cref{lem:surface_not_obstacle}.  
\end{proof}

\subsection{The Existence of a Structured Optimal Policy}\label{subsec:smaller_MDP}
This section shows the existence of one optimal policy that has a special structure. Intuitively, we show that there is an optimal policy that traverses between areas by moving from one corner to another. In designing the interpretable model, this will allow us to focus only on corners and ignore all other (infinite number of) states. 

Formally, we define a new MDP, $\cM'$. The states are all the corners, $\cC$ in the grid. The actions are ``go to adjacent corner.'' There are at most $2d$ actions from each state. The cost of action that moves from corner $c^1$ to an adjacent corner $c^2$ that differs on the $i$-th feature is $d(c^1,c^2)=|c_i^1-c_i^2|,$ which is exactly the cost to move from state $c^1$ to state $c^2.$

Denote by $V^*_{\cM'}:\cS\rightarrow\R$ the optimal value and by $\pi^*_{\cM'}$ the optimal policy for MDP $\cM'$. We will focus on the policy $\tilde{\pi}$ defined as \cref{alg:policy_pi_tilde}. The main theorem of this section is that $\tilde{\pi}$ is an optimal policy for $\cM$. To prove it, we first show that  the optimal value with respect to the full maze MDP, $\cM$, can be calculated using the optimal value for the smaller MDP, $\cM'$.   
\begin{algorithm}
\caption{Policy $\tilde{\pi}$}
\label{alg:policy_pi_tilde}
\begin{algorithmic}[1] 
\STATE \textbf{Input:} $s\in \cS$
\STATE \textbf{Output:}  actions
\IF {$s\in\cC$}
\STATE \textbf{return} $\pi^*_{\cM'}(s)$ 
\ELSE
\STATE  $c = \textrm{arg}\min_{c \in S(s)\cap\cC} d(s,c) + V^*_{\cM'}(c)$
\STATE \textbf{return}  $\texttt{DirectGoTo}(s,c)$
\ENDIF
\end{algorithmic}
\end{algorithm}

\begin{restatable}{lemma}{ValueLargeSmallMDP} \label{thm:rd_value_function_of_large_and_small_MDP}
For any state $s\in\cS$, it holds that 
$$V^*_{\cM}(s) = \min_{c \in S(s)\cap\cC} d(s,c) + V^*_{\cM'}(c).$$
\end{restatable}
This proves that indeed $\tilde{\pi}$ is an optimal policy. %
The proof of this lemma is in \cref{sec:analysis}.

\subsection{The Algorithm}\label{subsec:the_algorithm}
In this section, we show how to implement the optimal policy $\tilde{\pi}$ with an interpretable policy, i.e., a low depth decision tree. This interpretable policy is built in our main algorithm, Algorithm~\ref{alg:main_alg}. It is composed conceptually of the following steps.
(See Appendix~\ref{apx:algorithms}.)

We first compute the set of corners $\cC$ from the set of obstacles.
Then we build the MDP, $\cM'$, defined by the corners $\cC$, and compute an optimal policy and value function for $\cM'$. 
Finding the optimal policy $\pi^*_{\cM'}$ for the finite, smaller MDP can be done by any of the standard planning algorithms \citep{sutton2018reinforcement} in time $\tilde{O}(d|\cC|)$.\footnote{Note that standard algorithms cannot be implemented directly to find the optimal policy of the original MDP, $\cM$ as $\cM$ has infinite number of states. This adds another layer of complexity for finding the optimal interpretable policy for $\cM.$} Given the optimal policy $\pi^*_{\cM'}$ for $\cM'$, the algorithm then calls the \texttt{build\_grid\_tree} subroutine, which builds the decision tree.

The \texttt{build\_grid\_tree} subroutine partitions the state space such that at each part, all states have the same surface. 
The subroutine, for each feature $i$, builds a search tree over $L^i$ to find an index $j$ such that $L^i[j]\leq x_i<L^i[j+1]$. This decision tree has depth $O(\log|\cC|)$.
Special care is taken to ensure that in each leaf we have a single surface (details are in Appendix \ref{apx:algorithms}).
%
%
We then compute all the corners in each leaf (which is a single surface).
If we are already in a corner, we continue to the next corner that $\pi^*_{\cM'}$ defines, which we encode in the leaf.
Otherwise, we continue to the \texttt{build\_direction\_tree} subroutine.

The subroutine \texttt{build\_direction\_tree} receives as input the corners of the current leaf.
It chooses the best corner among them, where best  minimizes the sum of the minimal distance to the corner plus the minimal distance from the corner to the goal. (It is critical for our decision tree construction that the next state is a corner, which is guaranteed by  \cref{thm:rd_value_function_of_large_and_small_MDP}). 
Finally, we encode in the current leaf the identity of the best corner.


\paragraph{Analysis.}
%
One of the main challenges in the proof is proving \cref{thm:rd_value_function_of_large_and_small_MDP}.
The lemma shows that an optimal policy can move first to a corner and then continue as $\pi^*_{\cM'}$ indicates, i.e., moving only between corners. Put differently, the agent leaves each surface from a corner.
Next, we outline the high-level idea for proving this lemma. 

We show how to take an optimal policy $\pi^*_{\cM}$ to $\cM$ and modify it to ensure the agent leaves any surface from a corner. The modification is done in a \emph{backward manner} from the goal state to the starting state. Denote the current state by $s$. If the goal state is in the surface of $s$, i.e., $g\in S(s)$, the claim is immediate, as there is direct access between $s$ and $g$. Otherwise, focus on the backward path the agent takes from $s$ to $g$. Before reaching $g$, we pause this path at the last time the agent reaches a state $s'$ such that in some feature $i$, it holds that $s'_i \in L^i$. This $s'$ must exist because $g$ is not in the surface of $s$. 
We need to show that we can replace the path $s\rightarrow s'\rightarrow g$ by a path  $s\rightarrow s'\rightarrow c\rightarrow g$, where $c$ is a corner, and the cost of the two paths is identical. 

We consider a corner $c$, which is equal to $g$ at every feature except $i$, i.e., $c_j=g_j$ for all $j\in[d]-\{i\}$. In the $i$-th feature $c$ equals to $s'$, i.e.,  to $c_i=s'_i$. Notably, the corners $c$ and $g$ are both in the surface of $s'$. This means that the agent can freely go from $s'$ to $c$ and then to $g$. Note that this path has the same cost as the path induced by $\pi^*_{\cM}$ from $s'$ to $g$. The main benefit is that now the modified policy leaves from corner $c$. The proof can continue by induction where the new goal state is the corner $c$. 
(The full analysis and proof are in Appendix~\ref{sec:analysis}.)

\section{Conclusion and Open Problems}
In this paper, we found an optimal policy with complexity $O(2^d+\log k)$ for $(k,d)$-mazes, implying that for mazes, there is no price for interpretability. This result holds even in the face of uniform noise: at each state, the agent stays at the same place with some probability. 

Several problems are left open. Under what conditions can one reduce the complexity of the policy (or show a matching lower bound). In the paper we explored one type of noise. Nevertheless, building interpretable policies for other types of noise is of interest.  Lastly, can this result be generalized to different settings beyond mazes, allowing for optimal and interpretable policies.


\section*{Acknowledgements}
Yishay Mansour and Michal Moshkovitz has received funding from the European Research Council (ERC) under the European Union’s Horizon 2020 research and innovation program (grant agreement No. 882396), by the Israel Science Foundation (grant number 993/17), Tel Aviv University Center for AI and Data Science (TAD), and the Yandex Initiative for Machine Learning at Tel Aviv University. Cynthia Rudin acknowledges funding under NSF IIS: 2147061, NSF CCF: 1934964, and NIDA R01 DA054994. 

\newpage
\bibliographystyle{plainnat}
\bibliography{bib}

\appendix

\section{Algorithms}\label{apx:algorithms}
In this section we provide the pseudocode of the algorithms described throughout the paper.


\begin{algorithm}[ht]
\caption{Building Interpretable Optimal Policy}
\begin{algorithmic}[1] 
\STATE \textbf{Input:} lists $L^1,\ldots L^d$
\STATE \textbf{Output:}  interpretable policy 
\STATE \textcolor{blue}{\# Construct MDP $\cM'$ }
\STATE $\cC \leftarrow L^1 \times \ldots \times L^d$  \textcolor{blue}{\# The states of $\cM'$ }
\STATE $\cA' \leftarrow  \{\text{``go to } j \text{-th adjacent corner''}\}_j$ \textcolor{blue}{\# The actions of $\cM'$ }
\FOR {$c \in \cC$ }
\STATE \textcolor{blue}{\# Construct transition and cost functions of  $\cM'$ }
\STATE $j\leftarrow 0$ 
\FOR {$c'$ adjacent corner to $c$}

\STATE $\rT'(c,j) \leftarrow c'$ \textcolor{blue}{\# $j$th actions of $c$, moves agent to $c'$  }
\STATE $\rC'(c,j) \leftarrow d(c,c')$ \textcolor{blue}{\# the cost of this action  }
\STATE $j \leftarrow j+1$
\ENDFOR
\ENDFOR
\STATE $\cM' \leftarrow \langle\cC, \cA', \rT', \rC' \rangle$
\STATE $V^*_{\cM'},\pi^*_{\cM'}\leftarrow$ optimal value and policy for $\cM'$
\STATE \textbf{return} $\mathtt{build\_grid\_tree}(1,L^1,\ldots,L^d,V^*_{\cM'},\pi^*_{\cM'})$
\end{algorithmic}
\label{alg:main_alg}
\end{algorithm}

\begin{algorithm}[ht]
\caption{Build Grid Tree}
\begin{algorithmic}[1] 
\STATE \textbf{Input:} $i, L^1,\ldots, L^d, V^*_{\cM'},\pi^*_{\cM'}$ 
\STATE \textbf{Output:}  interpretable policy
\IF {$|L^i|>2$}
\STATE \textcolor{blue}{\# If current list is bigger $> 2$: add inner node and continue building the tree recursively }
\STATE $mid \leftarrow \lfloor\frac{|L^i|}{2}\rfloor$
\STATE $node.condition \leftarrow x_i\; ?\; L^i[mid]$ 
\STATE $node.less \leftarrow \mathtt{build\_grid\_tree} (i, L^1,\ldots L^{i-1}, L^{i}[1,\ldots,mid],L^{i+1},\ldots, L^d)$
\STATE $node.eq \leftarrow \mathtt{build\_grid\_tree} (i, L^1,\ldots L^{i-1}, [L^i[mid]],L^{i+1},\ldots, L^d )$
\STATE $node.large\leftarrow \mathtt{build\_grid\_tree} (i, L^1,\ldots L^{i-1}, L^{i}[mid,\ldots,|L^i|],L^{i+1},\ldots, L^d)$
\STATE \textbf{return} node
\ELSIF {$i<d$}
\STATE \textcolor{blue}{\# Alg. is done with the current list. If not the last list, continue to the next one}
\STATE \textbf{return} $\mathtt{build\_grid\_tree}(i+1, L^1,\ldots,L^d)$
\ELSE 
\STATE \textcolor{blue}{\# Now alg. can detect the surface}
\IF {$\forall i, |L^i|=1$ }
\STATE \textcolor{blue}{\# if a corner move to next corner, following policy $\pi^*_{\cM'}$}
\STATE $\textit{leaf.action} = \pi^*_{\cM'}(x)$
\STATE \textbf{return} \textit{leaf}
\ELSE
\STATE \textcolor{blue}{\# if not a corner, choose a corner to move to}
\STATE \textbf{return}
$\mathtt{build\_direction\_tree}( L^1\times\ldots\times L^d, V^*_{\cM'})$
\ENDIF 
\ENDIF
\end{algorithmic}
\label{alg:build_tree}
\end{algorithm}

\begin{algorithm}[ht]
\caption{Build Direction Tree}
\begin{algorithmic}[1] 
\STATE \textbf{Input:} $F, V^*_{\cM'}$ 
\STATE \textbf{Output:} internal node of the interpretable policy
\STATE \textcolor{blue}{\# choose lowest cost path along the special edges} 
\STATE \textcolor{blue}{\# exact direction depends on optimal corner among $F$}
\IF {$|F|=1$}
\STATE $c \leftarrow F[1]$ \textcolor{blue}{\# found optimal corner to move to}
\STATE $\textit{leaf.action} = \texttt{DirectGoTo}(x,c)$
\STATE \textbf{return} $\textit{leaf}$
\ELSE 
\STATE \textcolor{blue}{\# compare cost path of first and second corners in $F$}
\STATE $c^1, c^2 \leftarrow F[1], F[2]$ 
\STATE $\textit{node.condition} \leftarrow d(x,c^1) + V^*_{\cM'}(c^1) \leq d(x,c^2) + V^*_{\cM'}(c^2)$ 
\STATE $\textit{node.left} \leftarrow \mathtt{build\_direction\_tree}(F[1,3,\ldots])$ 
\STATE $\textit{node.right} \leftarrow \mathtt{build\_direction\_tree}(F[2,3,\ldots])$
\STATE \textbf{return} $node$
\ENDIF
\end{algorithmic}
\label{alg:build_direction_tree}
\end{algorithm}
\newpage

\section{Analysis}\label{sec:analysis}

We start with a simplification of \cref{thm:rd_value_function_of_large_and_small_MDP}, and show it holds for any corner. 
\begin{claim}\label{clm:small_mdp_value_similar_corner}
For any corner $c\in\cC$, it holds that  $V^*_{\cM}(c) = V^*_{\cM'}(c)$.  
\end{claim}
\begin{proof}
Fix an optimal policy $\pi^*_{\cM}$ for the original MDP $\cM$.
For any corner $c\in\cC$, the inequality $V^*_{\cM'}(c)\geq V^*_{\cM}(c)$ follows immediately as any trajectory in $\cM'$ corresponds to a trajectory in the original MDP with a similar total cost. 
To prove the interesting part,  $V^*_{\cM'}(c)\leq V_{\cM}^*(c)$, we will take the feature-wise point of view. Recall that for each feature $i\in[d],$ there is a list  $L^i\subseteq\R$ that is induced by the obstacles and goal state of the MDP $\cM$. A point is a corner, $c\in\cC$, if for each coordinate $c_i\in L^i$. 

For a trajectory to $g$, $x^0, x^1, ..., x^T=g$, which is induced by $\pi^*_{\cM}$, we say that a \emph{feature-event} happens at time $t<T$ if there is $i\in[d]$ such that $x^{t}_i\notin L^i$ and $x^{t+1}_i\in L^i$. We will prove the claim that for every corner $c\in\cC$ it holds that $V^*_{\cM}(c)=V^*_{\cM'}(c)$, by induction on the number of feature-events in the trajectory from $c$ to $g$. 
Once the agent reaches the goal state the number of feature-events is zero, which proves the basis of the induction. 

We turn to proving the induction step. Fix a corner $c\in\cC$ with $n>0$ feature-events. Focus on the first feature-event, at state $s$. So the path that $\pi^*_{\cM}$ induces is $c \longrightarrow s \longrightarrow g.$ Denote by $D\subseteq[d]$ the features that differ between $s$ and $c$ and by $i\in D$ the feature that caused the feature-event (i.e., $s_i\in L^i$ and for the state $x$  that is immediately before $s$, $x_i\notin L^i$). Denote by $c'$ the corner which is equal to $c$ at all features except $i$, i.e., for $\forall j\neq i$, $c'_j=c_j$ and in the $i$-th coordinate it is equal to $s_i$, i.e., $c'_i=s_i$. 

From the definition of feature-event $c,c'\in S(s).$  
From \cref{clm:direct_access} we know that $s$ can reach $c$ and $c'$ directly. 
 This implies that the trajectory $c \longrightarrow c' \longrightarrow  s \longrightarrow g$ has similar cost as the original trajectory from $c$ to $g$.
Number of feature-events for $c'$ is smaller than $n$. Thus, from the induction hypothesis, $V^*_{\cM}(c') = V^*_{\cM'}(c').$
We proved the claim as the following inequalities hold
\begin{align*}
    V^*_{\cM}(c) & = d(c,c') + V_{\cM}^*(c') \\
    & = d(c,c') + V_{\cM'}^*(c')
    \\
& \geq V_{\cM'}^*(c')
\end{align*}
\end{proof}

A similar proof applies to \cref{thm:rd_value_function_of_large_and_small_MDP}. For completeness, we repeat it here.  
\ValueLargeSmallMDP*
\begin{proof}
We generalize the definitions of $V^*_{\cM'}(s), V^*_{\cM}(s), \pi^*_{\cM}(s)$ to allow an arbitrary corner $c\in\cC$ to be a goal state and denote it as  $V^*_{\cM'}(s,c), V^*_{\cM}(s,c), \pi^*_{\cM}(s,c).$ Note that the two definitions coincide when $c=g$, i.e.,  $V^*_{\cM'}(s,g)=V^*_{\cM'}(s), V^*_{\cM}(s,g)=V^*_{\cM}(s),  \pi^*_{\cM}(s,g)=\pi^*_{\cM}(s).$ We will show a stronger claim: For any state $s\in\cS$ and a corner $c\in\cC$, it holds that 
$$V^*_{\cM}(s,c) = \min_{c' \in S(s)\cap\cC} d(s,c') + V^*_{\cM'}(c',c).$$

 To prove that for any state $s\in\cS$ and corners $c\in\cC, c' \in S(s)\cap\cC$ the inequality  $V^*_{\cM}(s,c) \leq d(s,c') + V^*_{\cM'}(c',c)$  holds, we will show a trajectory from $s$ to $c$ with cost $d(s,c') + V^*_{\cM'}(c',c)$. The trajectory is from $s$ to $c'$ and then from $c'$ to $c$. The cost of the first term is $d(s,c)$, from \cref{lem:surface_not_obstacle}, and the cost of the second term is $V^*_{\cM'}(c',c)$, by \cref{clm:small_mdp_value_similar_corner}.


To prove that the inequality $V^*_{\cM}(s,c) \geq \min_{c' \in S(s)\cap\cC} d(s,c') + V^*_{\cM'}(c',c)$ holds for any state $s\in\cS$ and a corner $c\in\cC$
we take the feature point of view and define the  \emph{feature-event}, similarly to the previous claim. For a trajectory $x^0, x^1, ..., x^T,$ a feature-event happens at time $t<T$ if there is $i\in[d]$ such that $x^{t}_i\notin L^i$ and $x^{t+1}_i\in L^i$.
We will prove the claim by induction on the number of feature events for any $s\in\cS,c\in\cC$ and a  trajectory between them defined by $\pi^*_{\cM}$. 
If the number of feature-events is zero, then the $s=c$ and $S(s)=c$, and the claim immediately follows as $V^*_{\cM}(c,c)=0=d(c,c) + V^*_{\cM'}(c,c)$. 

 To prove the claim by induction on the number of feature-events, we traverse the trajectory from $s$ to $c$ induced by $\pi^*_{\cM}(s,c)$, \emph{backwards} from $c$ to $s.$ Focus on the first (from last) feature-event, at state $s'$. So the trajectory that $\pi^*_{\cM}(s,c)$ induces, $p$, is $s \longrightarrow s' \longrightarrow c.$ Denote by $D\subseteq[d]$ the features that differ between $c$ and $s'$ and by $i\in D$ the feature that caused the feature-event. 

Denote by $c''$ the corner that is the same as $c$ except feature $i$ where it is equal to $s'$. I.e., for all $j\neq i$, $c''_j=c_j$ and $c''_i=s_i.$
We will show that the trajectory $p' = s \longrightarrow s' \longrightarrow  c'' \longrightarrow c$ has the same cost as the original trajectory between $s$ and $c.$ Since no feature-event happened between $s'$ and $c$, there is a direct access between them and  the cost from $s'$ to $c$ is equal to $d(s',c)=\sum_{j\in D}|s'_j-c_j|.$ Denote the cost from $s$ to $s'$, induced by $\pi^*_{\cM}$, by $V$. Then, the cost of $p$ is $V^*_{\cM}(s,c) = V + d(s',c) = V + d(s',c'')+d(c'',c).$
The cost of  $p'$ is $V+d(s',c'')+d(c'',c).$
\begin{align*}
V^*_{\cM}(s,c) & = V + d(s',c) \\
& = V  + d(s',c'') + d(c'',c) \\
& \geq V^*_{\cM}(s,c'') + d(c'',c) 
\end{align*}
The trajectory from $s$ to $c''$ has less feature-events. Thus,
from the induction hypothesis, $V^*_{\cM}(s,c'') =  \min_{c' \in S(s)\cap\cC} d(s,c') + V^*_{\cM'}(c',c'').$
\begin{align*}
V^*_{\cM}(s,c) & \geq \min_{c' \in S(s)\cap\cC} d(s,c') + V^*_{\cM'}(c',c'') + d(c'',c) \\
& \geq  \min_{c' \in S(s)\cap\cC} d(s,c') + V^*_{\cM'}(c',c),
\end{align*}
which is what we wanted to prove. 
\end{proof}

\begin{proof}(of main theorem) 
We need to show that $\tilde{\pi}$ is implemented by Algorithm~\ref{alg:main_alg}. Importantly, each leaf in the grid has the same surface. Thus, finding the best corner, as done in Algorithm~\ref{alg:build_direction_tree}, implements policy $\tilde{\pi}.$ 

The depth of the tree is bounded by the depth of finding the area, which is $O(d\log k)$, and finding the best corner. The depth of finding the best corner is the same as the number of corners in a surface. which is $\times_i L^i$ for the sets $L^i$ corresponding to the leaf with $|L^i|=1$ or $|L^i|=2.$
In the worst case, $|L^i|=2$ for all $i$, so the number of corners is bounded by $2^d$. However, if the starting state is sparse, the bound can be highly improved. Specifically, assume that the starting state has at most $n$ coordinates which are nonzero. Then, add $0$ to each each set $L^i$. And the number of corners of the starting state is only $2^n$. For sparse starting states, $2^n\ll 2^d.$ 

We remark that a generalization of the theorem to nonsparse vectors is possible in the following sense. Suppose there are $O(1)$ vectors, $\cV$, such that every starting state $s$, its Manhattan distance to some vector $v\in \cV$ is at most $n$. Then, results still hold, by adding $\cV$ to the $L^i$ lists. 
\end{proof}


\section{Uniform Noise}\label{sec:noise}
This section proves that adding noise to each state does not alter the optimal policy. Thus, it is enough to focus on a deterministic MDP, as this paper does.  
More formally, fix an MDP $\cM=\langle\cS, \cA, T, C \rangle$. 
The \emph{uniformly noisy MDP} is defined as $\cM^N=\langle\cS,\cA,P,C)$, where \[P(s'|s,a) = (1-\alpha)\cdot \mathbb{I}_{T(s)=s'} + \alpha \cdot \mathbb{I}_{s'=s},\] for some fixed $\alpha\in[0,1]$ and $\mathbb{I}$ is the indicator function. In words, for the uniform noisy MDP, there is a probability of $\alpha$ to stay at the same state and a probability of $1-\alpha$ to follow the deterministic MDP. Same $\alpha$ for every state.
We begin with a claim that will prove helpful later on; if all costs are multiplied by a scalar, the optimal policy does not change. 
\begin{claim}
Fix MDP $\cM=\langle\cS, \cA, T, C, g \rangle$. Suppose that all costs are multiplied by $c>0$. Then, optimal policy for $\cM^c=\langle\cS, \cA, T, c\cdot C, g \rangle$ does not change.
\end{claim}
\begin{proof}
Fix a policy $\pi$. Denote by $V^{\pi}$ the value function of $\pi$ for $\cM$ and by $V^{c,\pi}$ the value function of $\pi$ for $\cM^c$. Then, $V^{\pi,c}(s) = c\cdot V^{\pi}(s)$. An optimal policy is one that minimize the value function at each state and  $\argmin_{\pi} V^{c,\pi}(s) = \argmin_{\pi} V^\pi(s)$, thus proving the claim. 
\end{proof}
Now for the main claim for this section. 
\begin{lemma}
If $\pi$ is an optimal policy for $\cM$, then $\pi$ is an optimal policy for $\cM^N$.
\end{lemma}

\begin{proof}
 For every $s\in\cS$, denote by $s'=T(s)$. The Bellman equations for the noisy MDP are
\begin{align*}
 V^{N}(s) & = 1 + \alpha V^{N}(s) +
 (1-\alpha)V^{N}(s')\\
(1-\alpha) V^{N}(s)&= 1 + (1-\alpha)V^{N}(s').\\ 
V^{N}(s) &= 1/(1-\alpha) + V^{N}(s').   
\end{align*}
These are the same equations
for the original MDP besides the cost multiplied by $1/(1-\alpha)$, use previous claim to prove the lemma. 
\end{proof}

As a corollary we have the main theorem for the uniformly noisy setting.  
We are now ready to state the main theorem. 
\begin{theorem}
For every uniformly noisy $(k,d)$-maze there is an optimal interpretable policy $\pi^*:\cS \rightarrow \cA$ with complexity $O(\log k + 2^d)$, which implies that for a constant $d$ it is $O(\log k)$.
\end{theorem}
Similarly to the non-noisy setting, if the starting state is sparse with at most $n$ non-zero coordinates, then the depth of the policy is $O(\log k+2^n)$.

\end{document}